%% file: retrace_agent.tex
\documentclass{article}

\usepackage{iclr2018_conference,times}

\usepackage{graphicx}
\usepackage[utf8]{inputenc} 
\usepackage[T1]{fontenc}    
\usepackage{hyperref}       
\usepackage{url}            
\usepackage{booktabs}       
\usepackage{amsfonts}       
\usepackage{nicefrac}       
\usepackage{microtype}      
\usepackage{mathtools}
\usepackage{amsthm}
\usepackage{sidecap}
\usepackage{mathtools}
\usepackage{color, colortbl}
\usepackage{thmtools}
\usepackage{thm-restate}

\title{The Reactor: \\ A fast and sample-efficient Actor-Critic agent for 
Reinforcement Learning}

\iclrfinalcopy

\author{Audr{\=u}nas Gruslys, \\ DeepMind \\ \texttt{audrunas@google.com}
\And
Will Dabney, \\ DeepMind \\ \texttt{wdabney@google.com}
\And
Mohammad Gheshlaghi Azar, \\ DeepMind\\ \texttt{mazar@google.com}
\And
Bilal Piot, \\ DeepMind \\ \texttt{piot@google.com}
\And
Marc G. Bellemare, \\ Google Brain \\ \texttt{bellemare@google.com}
\And
R{\'e}mi Munos, \\ DeepMind \\ \texttt{munos@google.com}
}

\newcommand{\beq}{\begin{equation}}
\newcommand{\eeq}{\end{equation}}

\newcommand{\beqa}{\begin{eqnarray}}
\newcommand{\eeqa}{\end{eqnarray}}

\newcommand{\beqan}{\begin{eqnarray*}}
\newcommand{\eeqan}{\end{eqnarray*}}
\newcommand{\E}{{\mathbb E}}
\newcommand{\action}{{\hat a}}

\newcommand{\eqdef}{\stackrel{\rm def}{=}}

\newtheorem{assume}{Assumption}

\definecolor{reactor_color}{rgb}{0.874509804,0.784313725,0.831372549}
\definecolor{ac_color}{rgb}{0.835294118,0.890196078,0.788235294}
\definecolor{dqn_color}{rgb}{0.784313725,0.858823529,0.933333333}

\begin{document} 

\maketitle

\begin{abstract}
In this work, we present a new agent architecture, called Reactor, which 
combines multiple algorithmic and architectural contributions to produce 
an agent with higher sample-efficiency than Prioritized Dueling DQN 
\citep{wang2017sample} and Categorical DQN \citep{bellemare2017distributional}, 
while giving better run-time performance than A3C \citep{mnih2016asynchronous}. 
Our first contribution is a new policy evaluation algorithm called Distributional 
Retrace, which brings multi-step off-policy updates to the 
distributional reinforcement learning setting. 
The same approach can be used to convert several classes of 
multi-step policy evaluation algorithms, designed for expected value evaluation, 
into distributional algorithms. Next, we introduce the {\em $\beta$-leave-one-out} 
policy gradient algorithm, which improves the trade-off between variance 
and bias by using action values as a baseline. Our final algorithmic contribution is 
a new prioritized replay algorithm for sequences, which exploits the temporal locality of 
neighboring observations for more efficient replay prioritization. 
Using the Atari 2600 benchmarks, we show that each of these innovations contribute to 
both sample efficiency and final agent performance. Finally, we demonstrate that Reactor 
reaches state-of-the-art performance after 200 million frames and less than a day of training.
\end{abstract} 

\section{Introduction}
\label{introduction}

Model-free deep reinforcement learning has achieved several remarkable 
successes in domains ranging from super-human-level control in video games 
\citep{mnih15human} and the game of Go \citep{silver2016mastering, agzero}, to 
continuous motor control tasks \citep{lillicrap2015continuous,schulman2015trust}.

Much of the recent work can be divided into two categories. First, 
those of which that, often building on the DQN framework, 
act $\epsilon$-greedily according to an action-value function and train 
using mini-batches of transitions sampled from 
an experience replay buffer \citep{van2016deep,wang2015dueling,he2016learning,anschel2017averaged}. These \textit{value-function agents} benefit from 
improved sample complexity, but tend to 
suffer from long runtimes (e.g. DQN requires approximately a week to train on 
Atari). The second category are the
\textit{actor-critic agents}, which includes the asynchronous advantage actor-critic (A3C) algorithm,
introduced by \citet{mnih2016asynchronous}. These agents train on 
transitions collected by multiple actors running, 
and often training, in parallel \citep{schulman2017proximal,vezhnevets2017feudal}. The deep actor-critic agents train on each 
trajectory only once, and thus tend to have 
worse sample complexity. However, their distributed nature allows significantly 
faster training in terms of wall-clock time.
Still, not all existing algorithms can be put in the above two categories and 
various hybrid approaches do exist 
\citep{zhao2016deep,o2016combining,gu2016q,wang2017sample}.

Data-efficiency and off-policy learning are essential for many 
real-world domains where 
interactions with the environment are expensive. Similarly, wall-clock time 
(time-efficiency) directly impacts an algorithm's 
applicability through resource costs. The focus of this work is to produce an agent that 
is sample- and time-efficient. To this end, we introduce a new reinforcement learning agent, 
called {\em Reactor} (Retrace-Actor), which 
takes a principled approach to combining the sample-efficiency of off-policy 
experience replay with the time-efficiency 
of asynchronous algorithms. We combine recent advances in both 
categories of agents with novel contributions to produce an 
agent that inherits the benefits of both and reaches state-of-the-art 
performance over 57 Atari 2600 games.

Our primary contributions are (1) a novel policy gradient algorithm, 
$\beta$-LOO, which makes better use of action-value estimates to improve the 
policy gradient; (2) the first multi-step off-policy distributional reinforcement 
learning algorithm, distributional Retrace($\lambda$); (3) a novel prioritized 
replay for off-policy sequences of transitions; and (4) an optimized network and 
parallel training architecture.

We begin by reviewing background material, including relevant improvements to 
both value-function agents and actor-critic agents. In Section~\ref{sec:reactor} 
we introduce each of our primary contributions and present the Reactor agent. 
Finally, in Section~\ref{sec:results}, 
we present experimental results on the 57 Atari 2600 games from the Arcade 
Learning Environment (ALE) \citep{bellemare2013arcade}, as well as a series of 
ablation studies for the various components of Reactor.

\section{Background}
We consider a Markov decision process (MDP) with state space $X$ and finite 
action space $\mathcal{A}$. 
A (stochastic) policy $\pi(\cdot | x)$ is a mapping from states $x \in X$ to a 
probability distribution over 
actions. We consider a $\gamma$-discounted infinite-horizon criterion, with 
$\gamma \in [0, 1)$ the discount factor, and define for policy $\pi$ the action-value of 
a state-action pair $(x,a)$ as 
$$Q^{\pi}(x,a) \eqdef \E\Big[ {\textstyle \sum_{t\geq 0} } \gamma^t r_t | x_0=x, a_0=a, \pi\Big],$$
where $(\{x_t\}_{t\geq 0})$ is a trajectory generated by choosing $a$ in $x$ 
and following $\pi$ thereafter, i.e., $a_t\sim \pi(\cdot|x_t)$ (for $t\geq 1$), 
and $r_t$ is the reward signal. The objective in reinforcement learning 
is to find an optimal policy $\pi^*$, which maximises $Q^{\pi}(x,a)$. 
The optimal action-values are given by $Q^*(x,a)=\max_{\pi} Q^{\pi}(x,a)$. 


\subsection{Value-based algorithms}

The Deep Q-Network (DQN) framework, introduced by \citet{mnih15human}, 
popularised the current line of research into deep reinforcement learning by reaching 
human-level, and beyond, performance across 57 Atari 2600 games in the ALE. 
While DQN includes many specific components, the essence of the framework, much 
of which is shared by Neural Fitted Q-Learning \citep{riedmiller2005neural}, is 
to use of a deep convolutional neural network to approximate an action-value 
function, training this approximate action-value function using the 
Q-Learning algorithm \citep{watkins1992} and mini-batches of one-step 
transitions ($x_t, a_t, r_t, x_{t+1}, \gamma_t$) drawn randomly from an 
experience replay buffer \citep{lin1992self}. Additionally, the next-state 
action-values are taken from a \textit{target network}, which is updated to 
match the current network periodically. Thus, the temporal difference (TD) error 
for transition $t$ used by these algorithms is given by
\beqa\label{eq:tderr}
\delta_t = r_t + \gamma_t \max_{a' \in \mathcal{A}} Q(x_{t+1}, a'; \bar \theta) - Q(x_t, 
a_t; \theta),
\eeqa
where $\theta$ denotes the parameters of the network and $\bar \theta$ are the 
parameters of the target network.

Since this seminal work, we have seen numerous extensions and improvements that 
all share the same underlying framework. Double DQN \citep{van2016deep}, 
attempts to correct for the over-estimation bias inherent in Q-Learning by 
changing the second term of \eqref{eq:tderr} to $Q(x_{t+1}, \arg\max_{a' \in \mathcal{A}} 
Q(x_{t+1}, a'; \theta); \bar \theta)$. The dueling architecture 
\citep{wang2015dueling}, changes the network to estimate 
action-values using separate network heads $V(x; \theta)$ and $A(x, a; \theta)$ with
$$Q(x, a; \theta) = V(x; \theta) + A(x, a; \theta) - \frac{1}{|A|} \sum_{a'} A(x, a'; \theta).$$

Recently, \citet{rainbow} introduced Rainbow, a value-based 
reinforcement learning agent combining many of these improvements into a single 
agent and demonstrating that they are largely complementary. Rainbow significantly 
out performs previous methods, but also inherits the poorer time-efficiency of the DQN framework.
We include a detailed comparison between Reactor and Rainbow in the Appendix. 
In the remainder of the section we will describe in more depth other recent 
improvements to DQN.

\subsubsection{Prioritized experience replay}
The experience replay buffer was first introduced by \citet{lin1992self} and 
later used in DQN \citep{mnih15human}. Typically, the replay buffer is 
essentially a first-in-first-out queue with new transitions gradually replacing 
older transitions. The agent would then sample a mini-batch uniformly at random 
from the replay buffer. Drawing inspiration from prioritized sweeping 
\citep{moore1993prioritized}, prioritized experience replay replaces the uniform 
sampling with prioritized sampling proportional to the absolute TD error 
\citep{schaul16prioritized}.

Specifically, for a replay buffer of size $N$, prioritized experience replay 
samples transition $t$ with probability $P(t)$, and applies weighted 
importance-sampling with $w_t$ to correct for the prioritization bias, where
\begin{equation}
\label{eq-importance}
P(t) = \frac{p_t^\alpha}{\sum_k p_k^\alpha},\ \quad w_t = \left( \frac{1}{N} 
\cdot \frac{1}{P(t)} \right)^\beta,\ \quad p_t = |\delta_t| + \epsilon,\ \quad 
\alpha, \beta, \epsilon > 0.
\end{equation}

Prioritized DQN significantly increases both the sample-efficiency and final 
performance over DQN on the Atari 2600 benchmarks \citep{schaul2015prioritized}.

\subsubsection{Retrace($\lambda$)}\label{sec:retrace}
Retrace($\lambda$) is a convergent off-policy multi-step algorithm extending the 
DQN agent \citep{munos2016safe}. Assume that some trajectory $\{x_0, a_0, r_0, x_1, a_1, 
r_1, \dots, x_t, a_t, r_t, \dots, \}$ has been generated according to {\em 
behaviour policy} $\mu$, i.e., $a_t\sim \mu(\cdot|x_t)$. Now, we aim to 
evaluate the value of a different {\em target policy} $\pi$, i.e.~we want to 
estimate 
$Q^{\pi}$. The Retrace algorithm will update our current estimate $Q$ of 
$Q^{\pi}$ in the direction of
\beqa\label{eq:retrace}
\Delta Q(x_t, a_t)\eqdef {\textstyle \sum_{s\geq t}} \gamma^{s-t} (c_{t+1}\dots 
c_s) 
\delta^{\pi}_s Q ,
\eeqa
where $\delta^{\pi}_s Q \eqdef  r_s+\gamma \E_{\pi} [Q(x_{s+1},\cdot)] - 
Q(x_s,a_s)$ is the temporal difference at time $s$ under $\pi$, 
and 
\beq \label{eq:trace.cut}
c_s=\lambda\min\big(1,\rho_s\big),\ \quad \rho_s = 
\frac{\pi(a_s|x_s)}{\mu(a_s|x_s)}.
\eeq 
The Retrace algorithm comes with the theoretical guarantee that in finite state 
and action spaces, repeatedly updating our current estimate $Q$ according to 
(\ref{eq:retrace}) produces a sequence of Q functions which converges to 
$Q^{\pi}$ for a fixed 
$\pi$ or to $Q^*$ if we consider a sequence of policies $\pi$ which become 
increasingly greedy w.r.t.~the $Q$ estimates \citep{munos2016safe}.

\subsubsection{Distributional RL}\label{sec:distrl}
Distributional reinforcement learning refers to a class of algorithms that 
directly estimate the distribution over returns, 
whose expectation gives the traditional value function 
\citep{bellemare2017distributional}. Such approaches can be made 
tractable with a distributional Bellman equation, and the recently proposed 
algorithm $C51$ showed state-of-the-art performance 
in the Atari 2600 benchmarks. $C51$ parameterizes the distribution over returns 
with a mixture over Diracs centered on a uniform grid,
\begin{equation}
Q(x, a; \theta) = \sum_{i=0}^{N-1} q_i(x, a; \theta) z_i,\ \quad q_i = 
\frac{e^{\theta_i(x, a)}}{\sum_{j=0}^{N-1} e^{\theta_j(x, a)}},\ \quad z_i = v_{\min} + i 
\frac{v_{\max} - v_{\min}}{N-1},
\end{equation}
with hyperparameters $v_{\min} , v_{\max}$ that bound the distribution support of size $N$.

\subsection{Actor-critic algorithms}\label{sec:ac}
In this section we review the actor-critic framework for reinforcement learning 
algorithms and then discuss 
recent advances in actor-critic algorithms along with their various trade-offs. 
The asynchronous advantage actor-critic (A3C) algorithm 
\citep{mnih2016asynchronous}, maintains a parameterized policy $\pi( a | x; 
\theta)$ and value function $V(x; \theta_v)$, which are updated with
\begin{eqnarray}
\triangle \theta = \nabla_\theta \log \pi(a_t | x_t; \theta) A(x_t, a_t;
\theta_v),\ \quad \triangle \theta_v = A(x_t, a_t; \theta_v) 
\nabla_{\theta_v} V(x_t),\\
\text{where, }\ \quad A(x_t, a_t; \theta_v) = \sum_k^{n-1} \gamma^k 
r_{t+k} + \gamma^n V(x_{t+n}) - V(x_t).\label{pgadv}
\end{eqnarray}
A3C uses $M = 16$ parallel CPU workers, each acting independently in the 
environment and applying the above updates asynchronously to a shared set of 
parameters. In contrast to the previously discussed value-based methods, A3C is 
an on-policy algorithm, and does not use a GPU nor a replay buffer. 

Proximal Policy Optimization (PPO) is a closely related actor-critic algorithm 
\citep{schulman2017proximal}, which replaces the advantage \eqref{pgadv} with,
$$\min(\rho_t A(x_t, a_t; \theta_v), clip(\rho_t, 1 - 
\epsilon, 1 + \epsilon) A(x_t, a_t; \theta_v)),\ \epsilon > 0,$$
where $\rho_t$ is as defined in Section~\ref{sec:retrace}. Although both PPO and 
A3C run $M$ parallel workers collecting 
trajectories independently in the environment, PPO collects these experiences to 
perform a single, synchronous, update in contrast 
with the asynchronous updates of A3C. 

Actor-Critic Experience Replay (ACER) extends the A3C framework with an 
experience replay buffer, Retrace algorithm for off-policy corrections, and 
the Truncated Importance Sampling Likelihood Ratio (TISLR) algorithm 
used for off-policy policy optimization \citep{wang2017sample}.

\section{The Reactor}\label{sec:reactor}
The Reactor is a combination of four novel contributions on top of recent 
improvements to both 
deep value-based RL and policy-gradient algorithms. 
Each contribution moves Reactor towards our goal of achieving both sample and time efficiency.

\subsection{$\beta$-LOO}\label{sec:betaloo}

The Reactor architecture represents both a policy $\pi(a|x)$ and action-value function 
$Q(x,a)$. We use a policy gradient algorithm to train the actor $\pi$ which 
makes use of our current estimate $Q(x,a)$ of $Q^{\pi}(x,a)$.  
Let $V^{\pi}(x_0)$ be the value function at some initial state $x_0$, the 
policy gradient theorem says that $\nabla 
V^{\pi}(x_0) = \E\big[\sum_t \gamma^t \sum_a Q^{\pi}(x_t,a) \nabla \pi(a|x_t) 
\big]$, where $\nabla$ refers to the gradient w.r.t.~policy parameters \citep{Sutton00policygradient}. We now 
consider several possible ways to estimate this gradient.

To simplify notation, we drop the dependence on the state $x$ for now and 
consider the problem of estimating the quantity
\beq\label{eq:G}
G = {\textstyle \sum_a } Q^{\pi}(a) \nabla \pi(a).
\eeq

In the off-policy case, we consider estimating $G$ using a single action $\action$ 
drawn from a (possibly different from $\pi$) behaviour distribution 
$\action\sim\mu$. Let us assume that for the chosen action $\action$ we have access to an 
unbiased estimate $R(\action)$ of $Q^{\pi}(\action)$. Then, we can use likelihood ratio (LR) 
method combined with an importance sampling (IS) ratio (which we call ISLR) to 
build an unbiased estimate of $G$:
$$ \hat G_{\mbox{\tiny ISLR}} = \frac{\pi(\action)}{\mu(\action)}(R(\action) - V) 
\nabla\log\pi(\action) ,$$
where $V$ is a baseline that depends on the state but not on the chosen action. 
However this estimate suffers from high variance. A possible way for reducing 
variance is to estimate $G$ directly from (\ref{eq:G}) by using the return 
$R(\action)$ for the chosen action $\action$ and our current estimate $Q$ of $Q^{\pi}$ for 
the other actions, which leads to the so-called {\em leave-one-out} (LOO) 
policy-gradient estimate:
\beq\label{eq:LOO}
\hat G_{\mbox{\tiny LOO}} = R(\action) \nabla \pi(\action) + {\textstyle \sum_{a\neq \action} } 
Q(a) \nabla 
\pi(a) .
\eeq
This estimate has low variance but may be biased if the estimated $Q$ values 
differ from $Q^{\pi}$. A better bias-variance tradeoff may be obtained by 
the more general {\em $\beta$-LOO policy-gradient} estimate:
\beq\label{eq:beta.LOO}
\hat G_{\mbox{\tiny $\beta$-LOO}} = \beta (R(\action) - Q(\action)) \nabla \pi(\action) + 
{\textstyle \sum_{a} }
 Q(a) \nabla \pi(a),
\eeq
where $\beta =  \beta(\mu,\pi,\action)$ can be a function of both policies, $\pi$ and 
$\mu$, and the selected action $\action$. Notice that when $\beta=1$, 
(\ref{eq:beta.LOO}) reduces to (\ref{eq:LOO}), and when $\beta=1/\mu(\action)$, then 
(\ref{eq:beta.LOO}) is
\beq \label{eq:1/mu.loo}
\hat G_{\mbox{\tiny $\frac{1}{\mu}$-LOO}} = 
\frac{\pi(\action)}{\mu(\action)} (R(\action) - Q(\action)) \nabla\log\pi(\action) + {\textstyle \sum_a } 
Q(a) 
\nabla\pi(a).
\eeq
This estimate is unbiased and can be seen as a generalization of $\hat 
G_{\mbox{\tiny ISLR}}$ where instead of using a state-only dependent baseline, 
we use a state-and-action-dependent baseline (our current estimate $Q$) and add 
the correction term $\sum_a\nabla\pi(a)Q(a)$ to cancel the bias. 
Proposition~\ref{prop:bias} gives our analysis of the bias of $G_{\mbox{\tiny 
$\beta$-LOO}}$, with a proof left to the Appendix.
\begin{restatable}{proposition}{propbias}\label{prop:bias}
Assume $\action\sim \mu$ and that $\E[R(\action)]=Q^{\pi}(\action)$. Then, 
the bias of $G_{\mbox{\tiny $\beta$-LOO}}$ is $\big| \sum_a (1-\mu(a)\beta(a)) 
\nabla\pi(a) [Q(a)-Q^{\pi}(a)]\big|$.
\end{restatable}
Thus the bias is small when $\beta(a)$ is close to $1/\mu(a)$, or 
when the $Q$-estimates are close to the true $Q^{\pi}$ values, and unbiased
regardless of the estimates if $\beta(a) = 1/\mu(a)$. The variance 
is low when $\beta$ is small, therefore, in order to improve the bias-variance tradeoff 
we recommend using the $\beta$-LOO estimate with $\beta$ defined as:
$\beta(\action) = \min\big( c, \frac{1}{\mu(\action)}\big),$ for some constant
$c\geq 1$. This truncated $1/\mu$ coefficient shares similarities with the truncated IS 
gradient estimate introduced in \citep{wang2017sample} (which we call TISLR for 
truncated-ISLR):
\beqan
\hat G_{\mbox{\tiny TISLR}}\! = \!\min\Big(c, \frac{\pi(\action)}{\mu(\action)}\Big) 
(R(\action) - V) \nabla\log\pi(\action)  \! + \! \sum_a \Big( \frac{\pi(a)}{\mu(a)}-c 
\Big)_{\!\! +} \mu(a) 
(Q^{\pi}(a) - V) \nabla\log\pi(a) .
\eeqan

The differences are: (i) we truncate $1/\mu(\action) = \pi(\action)/\mu(\action)\times 1/\pi(\action)$ 
instead of truncating $ \pi(\action)/\mu(\action)$, which provides an additional variance 
reduction due to the variance of the LR 
$\nabla\log\pi(\action)=\frac{\nabla\pi(\action)}{\pi(\action)}$ (since this LR may be large when 
a low probability action is chosen), and (ii) we use our $Q$-baseline instead 
of a $V$ baseline, reducing further the variance of the LR estimate.

\subsection{Distributional Retrace}\label{sec:dist_retrace}
\begin{figure}
\centering
\includegraphics[width=1.0\textwidth]{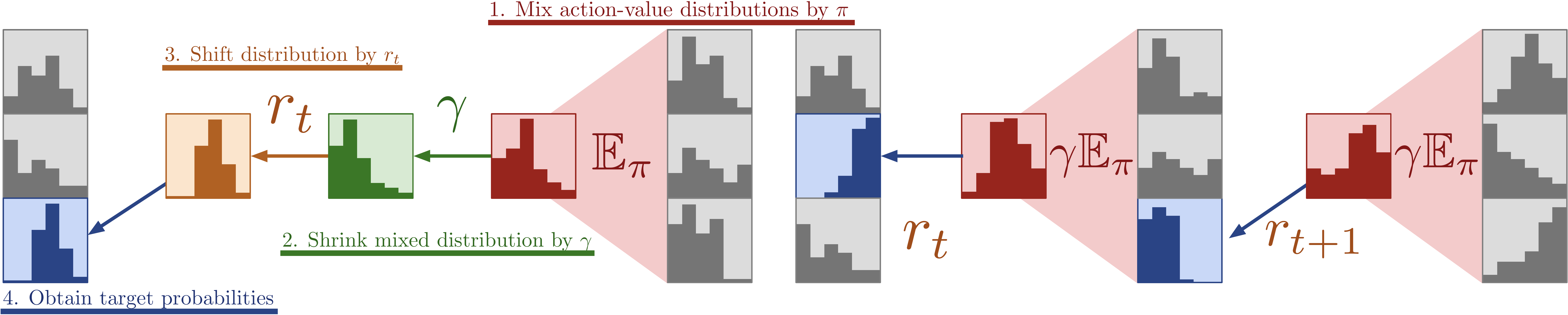}
  \caption{Single-step (left) and multi-step (right) distribution bootstrapping.}
  \label{distributions-multi-step}
\end{figure}
In off-policy learning it is very difficult to produce an 
unbiased sample $R(\action)$ of $Q^{\pi}(\action)$ when following another policy $\mu$. 
This would require using full importance sampling correction along the 
trajectory.  Instead, we use the off-policy corrected return computed 
by the Retrace algorithm, which produces a (biased) estimate of $Q^{\pi}(\action)$ 
but whose bias vanishes asymptotically \citep{munos2016safe}.

In Reactor, we consider predicting an 
approximation of the return distribution function from any state-action pair 
$(x,a)$ in a similar way as in \citet{bellemare2017distributional}. The 
original 
algorithm C51 described in that paper considered single-step Bellman updates 
only. Here we need to extend this idea to multi-step updates 
and handle the off-policy correction performed by the Retrace algorithm, as 
defined in
\eqref{eq:retrace}. Next, we describe these two extensions.

\paragraph{Multi-step distributional Bellman operator:} First, we extend C51 to 
multi-step Bellman backups. We consider return-distributions from $(x,a)$ of the form 
$\sum_{i} q_i(x,a) \delta_{z_i}$ (where $\delta_z$ denotes a Dirac in $z$) 
which are supported on a finite uniform grid $\{z_i\}\in [v_{\min}, 
v_{\max}]$, $z_i < z_{i+1}$, $z_{1} = v_{\min}$, $z_{m} = v_{\max}$. 
The coefficients $q_i(x,a)$ (discrete distribution) corresponds to the 
probabilities assigned to each atom 
$z_i$ of the grid. From an observed $n$-step sequence $\{x_t,a_t, r_t, x_{t+1}, 
\dots, x_{t+n}\}$, generated by behavior policy $\mu$ 
(i.e, $a_s\sim\mu(\cdot|x_s)$ for $t \leq s<t+n$), we build the 
$n$-step backed-up return-distribution from $(x_t, a_t)$. The $n$-step distributional 
Bellman target, whose expectation is $\sum_{s=t}^{t+n-1} \gamma^{s-t} r_s + \gamma^n Q(x_{t+n},a)$, 
is given by:
$$\sum_{i} q_i(x_{t+n},a) \delta_{z_i^{n}}, \mbox{ with } z^{n}_i = 
\sum_{s=t}^{t+n-1}\gamma^{s-t} r_s + \gamma^n z_i.$$
Since this distribution is supported on the set of atoms $\{z^{n}_i\}$, which is 
not necessarily aligned with the grid 
$\{z_i\}$, we do a projection step and minimize the KL-loss between the projected 
target and the current estimate, just as with C51 except with a different target distribution \citep{bellemare2017distributional}.

\paragraph{Distributional Retrace:} Now, the Retrace algorithm defined in 
\eqref{eq:retrace} involves an off-policy correction which is not handled by the 
previous $n$-step distributional Bellman backup. The key to extending this 
distributional back-up to off-policy learning is to rewrite the Retrace algorithm as a linear 
combination 
of $n$-step Bellman backups, weighted by some coefficients $\alpha_{n,a} $. 
Indeed, notice that \eqref{eq:retrace} rewrites as
\beqan
\Delta Q(x_t, a_t) = \sum_{n\geq 1} \sum_{a\in \mathcal{A}} \alpha_{n,a} \Big[ 
\underbrace{ \sum_{s=t}^{t+n-1} 
\gamma^{s-t} r_s + \gamma^n Q(x_{t+n},a)}_{n\mbox{\tiny -step Bellman backup}} 
\Big] - Q(x_t, a_t),
\eeqan
where $\alpha_{n,a} = \big(c_{t+1}\dots c_{t+n-1}\big) \big( \pi(a|x_{t+n}) - 
{\mathbb I}\{a=a_{t+n}\} c_{t+n}\big)$. These coefficients depend on the degree 
of off-policy-ness (between $\mu$ and $\pi$) along the trajectory. We have that 
$\sum_{n\geq 1}\sum_a \alpha_{n,a} = \sum_{n\geq 1} \big(c_{t+1}\dots c_{t+n-1}\big) (1 - 
c_{t+n}) = 1$, but notice some coefficients may be negative. However, in 
expectation (over the behavior policy) they are non-negative. Indeed, 
\beqan
\E_{\mu}[\alpha_{n,a}] &=& \E\Big[\big(c_{t+1}\dots c_{t+n-1}\big) 
\E_{a_{t+n}\sim\mu(\cdot|x_{t+n})}\big[ \pi(a|x_{t+n}) - {\mathbb 
I}\{a=a_{t+n}\}c_{t+n} |x_{t+n}\big] \Big] \\
&=& \E\Big[\big(c_{t+1}\dots c_{t+n-1}\big)  \Big( \pi(a|x_{t+n}) - 
\mu(a|x_{t+n}) \lambda \min\big(1, 
\frac{\pi(a|x_{t+n})}{\mu(a|x_{t+n})}\big)\Big) 
\Big]\geq 0,
\eeqan
by definition of the $c_s$ coefficients \eqref{eq:trace.cut}. 
Thus in expectation (over the behavior policy), the Retrace update can be seen 
as a {\em convex} combination of $n$-step Bellman updates.

Then, the distributional Retrace algorithm can be defined as backing up a {\em 
mixture} of $n$-step distributions. More precisely, 
we define the Retrace target distribution as:
$$\sum_{i=1} q^{*}_i(x_t,a_t) \delta_{z_i}, \mbox{ with } q^{*}_i(x_t,a_t) = 
\sum_{n\geq 1}\sum_a \alpha_{n,a} \sum_{j} q_j(x_{t+n},a_{t+n}) h_{z_i}(z^{n}_j),$$
where $h_{z_i}(x)$ is a linear interpolation kernel, projecting onto the support $\{ z_i \}$: 
$$h_{z_i}(x) = \left\{
  \begin{array}{lr}
   (x - z_{i-1}) / (z_{i} - z_{i-1}), & \text{ if } z_{i-1} \leq x \leq z_{i}\\
   (z_{i+1} - x) / (z_{i+1} - z_{i}), & \text{ if } z_{i} \leq x \leq z_{i+1}\\
    0, & \text{ if } x\leq z_{i-1} \text{ or } x \geq z_{i+1}\\
    1, & \text{ if } (x\leq v_{\min} \text{ and } z_i = v_{\min}) \text{ or } 
                    (x\geq v_{\max} \text{ and } z_i = v_{\max})\\
  \end{array}\right\}$$
We update the current 
probabilities $q(x_t, a_t)$ by performing a gradient step on the KL-loss 
\beq\label{eq:kl.gradient}
\nabla \mbox{KL}(q^{*}(x_t,a_t), q(x_t,a_t)) = -\sum_{i=1} 
q^{*}_i(x_t,a_t)\nabla 
\log q_i(x_t,a_t).
\eeq
Again, notice that some target ``probabilities'' $q^{*}_i(x_t,a_t)$ may be negative 
for some sample trajectory, but in expectation they will be non-negative. 
Since the gradient of a KL-loss is linear w.r.t.~its first argument, our update 
rule~\eqref{eq:kl.gradient} provides an unbiased estimate of the gradient of the 
KL between the expected (over the behavior policy) Retrace target distribution 
and the current predicted distribution.\footnote{We store past action probabilities $\mu$ together with actions taken in the replay memory.}

\paragraph{Remark:} The same method can be applied to other algorithms (such as 
TB($\lambda$) \citep{precup2000eligibility} and importance sampling 
\citep{precup01offpolicy}) in order to derive distributional versions of other
off-policy multi-step RL algorithms.  

\subsection{Prioritized sequence replay}\label{sec:prioritized_seq}
Prioritized experience replay has been shown to boost both statistical 
efficiency and final performance of deep RL agents 
\citep{schaul16prioritized}. However, as originally defined prioritized replay 
does not handle sequences of transitions and weights all unsampled transitions 
identically. In this section we present an alternative initialization strategy, 
called \textit{lazy initialization}, 
and argue that it better encodes prior information about temporal difference 
errors. We then briefly describe our computationally efficient prioritized sequence 
sampling algorithm, 
with full details left to the appendix.

It is widely recognized that TD errors tend to be temporally correlated, indeed 
the need to 
break this temporal correlation has been one of the primary justifications for 
the use of experience replay \citep{mnih15human}. Our proposed algorithm begins 
with this fundamental assumption.
\begin{assume}
Temporal differences are temporally correlated, with correlation decaying 
on average with the time-difference between two transitions.
\end{assume}
Prioritized experience replay adds new transitions to the replay buffer with a 
constant priority, but given the above assumption we can devise a better 
method. Specifically, we propose to add experience to the buffer with 
\textit{no 
priority}, inserting a priority only after the transition has been 
sampled and used for training. Also, instead of sampling transitions, we 
assign priorities to all (overlapping) sequences of length $n$. When sampling, 
sequences with an assigned 
priority are sampled proportionally to that priority. Sequences with no 
assigned 
priority are sampled proportionally to the average priority of 
assigned priority sequences within some local neighbourhood. Averages are 
weighted to compensate for sampling biases (i.e. more samples 
are made in areas of high estimated priorities, and in the absence of 
weighting this would lead to overestimation of unassigned priorities).

The \textit{lazy initialization} 
scheme starts with priorities $p_{t}$ corresponding to the sequences
$\{x_t, a_t, \ldots, x_{t+n}\}$ for which a priority was already 
assigned. Then it extrapolates a priority of all other sequences in the 
following way. Let us define a partition $(I_i)_i$ of the states ordered by 
increasing time such that each cell $I_i$ contains exactly one state $s_i$ with 
already assigned priority. We define the estimated priority $\hat p_{t}$ to 
all other sequences as $\hat p_{t} = 
\sum_{s_i \in J(t)} \frac{w_i} {\sum_{i'\in J(t)} w_{i'}} p(s_i)$, where 
$J(t)$ is a collection of contiguous cells $(I_i)$ containing time 
$t$, and $w_i= |I_{i}|$ is the length of the cell $I_{i}$ containing $s_i$.
For already defined priorities denote $\hat p_t=p_t$. 
%
Cell sizes work as estimates of inverse local density and are 
used as importance weights for priority estimation.
\footnote{Not to be confused with importance weights of produced samples.} 
For the algorithm to be unbiased, partition $(I_i)_i$ must 
{\bf not} be a function of the assigned priorities.
So far we have defined a class of algorithms all free to choose the 
partition $(I_i)$ and the collection of cells $I(t)$, as long that they satisfy 
the above constraints. Figure 
\ref{lazy-prioritization} in the Appendix illustrates the above description.

Now, with probability 
$\epsilon$ we sample uniformly at random, and with probability $1 - \epsilon$ 
we sample proportionally to $\hat p_{t}$.
We implemented an algorithm 
satisfying the above constraints and called it \textit{Contextual Priority 
Tree} (CPT). It is based on AVL trees 
\citep{velskii1976avl} and can execute sampling, insertion, deletion and 
density evaluation in $O(\ln (n))$ time. We describe CPT in detail 
in the Appendix in Section \ref{app-tree-prioritized}.

We treated prioritization as purely a variance reduction technique. 
Importance-sampling weights were evaluated as in prioritized 
experience replay, with fixed $\beta = 1$ in (\ref{eq-importance}). We 
used simple gradient magnitude estimates as priorities, corresponding to a mean 
absolute TD error along a sequence for Retrace, as defined in (\ref{eq:retrace}) 
for the classical RL case, and total variation 
in the distributional Retrace case.\footnote{Sum of absolute discrete probability differences.}

\subsection{Agent architecture}\label{sec:agent_arch}
\begin{figure}[t]
\centering
\vspace{0pt}
\begin{minipage}{.35\textwidth}
\hspace{-20px}  \includegraphics[width=1.15\textwidth]{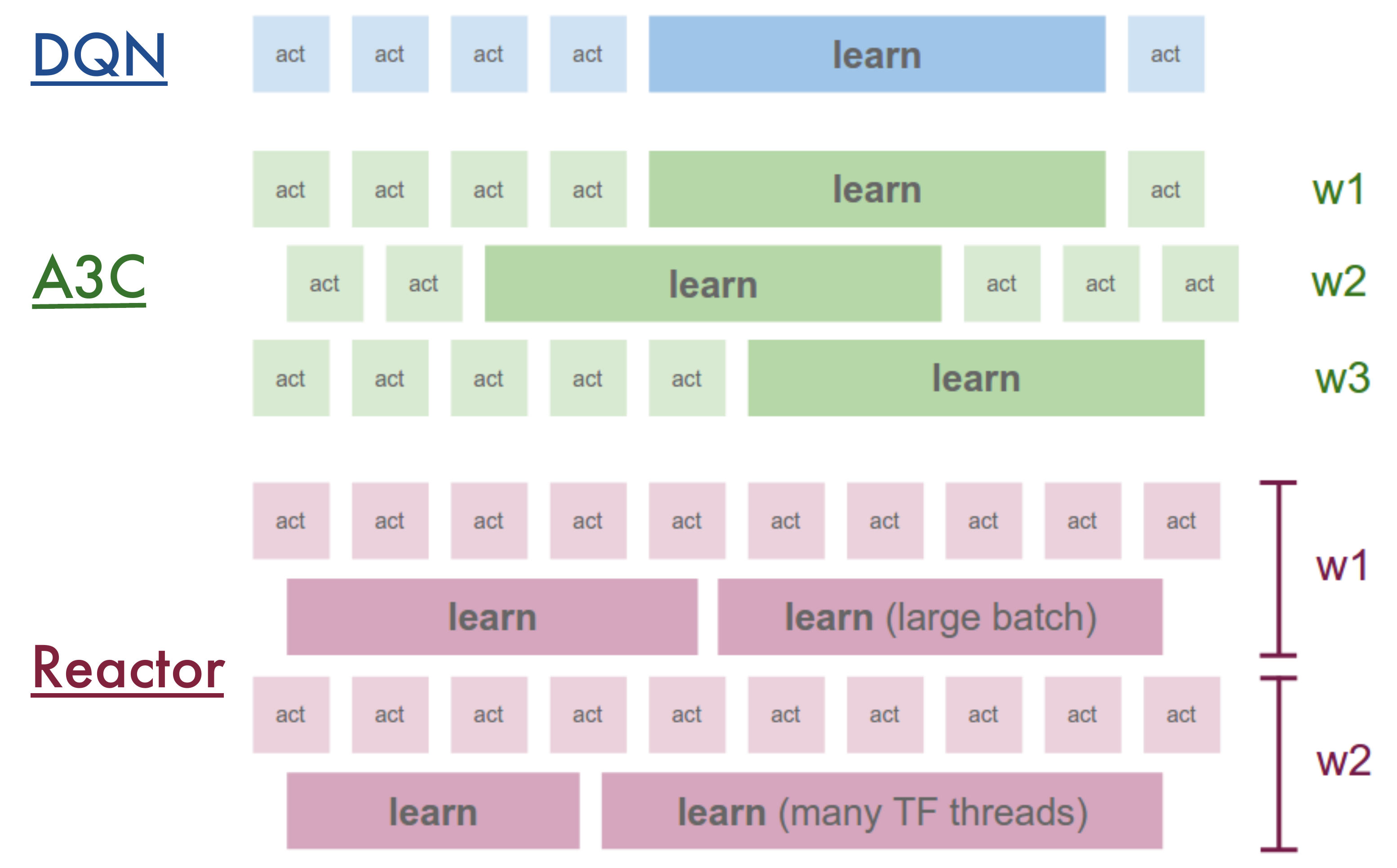}
\end{minipage}
\begin{minipage}{.01\textwidth}
\hspace{2px}
\end{minipage}
\vspace{0pt}
\begin{minipage}{.55\textwidth}
\begin{tabular}{| l | c | c | l |}
\hline
Algorithm & Training Time & Type & \# Workers \\
\hline
\hline
\cellcolor{dqn_color} DQN & 8 days & GPU & 1  \\
\cellcolor{dqn_color}  Double DQN & 8 days & GPU & 1  \\
\cellcolor{dqn_color} Dueling & 8 days & GPU & 1  \\
\cellcolor{dqn_color} Prioritized DQN & 8 days & GPU & 1  \\
\cellcolor{dqn_color} Rainbow & 10 days & GPU & 1  \\
\cellcolor{ac_color} A3C & 4 days & CPU & 16 \\
\cellcolor{reactor_color} Reactor & {\bf < 2 days} & CPU & 10+1 \\
\cellcolor{reactor_color} Reactor 500m & 4 days & CPU & 10+1 \\
\cellcolor{reactor_color} Reactor* & {\bf < 1 day} & CPU & 20+1 \\
\hline
\end{tabular}
\end{minipage}%
  \caption{\small (Left) The model of parallelism of DQN, A3C and Reactor 
    architectures. Each row represents a separate thread. In Reactor's case,
    each worker, consiting of a learner and an actor is run on a separate 
    worker machine. (Right) Comparison of training times and resources for 
    various algorithms. 500m denotes 500 million training frames; otherwise 
    200m training frames were used.}
  \label{parallelism}
\end{figure}

In order to improve CPU utilization we decoupled acting from learning. 
This is an important aspect of our architecture: an {\em acting thread} 
receives observations, submits actions to the environment, and stores 
transitions in memory, while a {\em learning thread} re-samples 
sequences of experiences from memory and trains on them (Figure 
\ref{parallelism}, left). We typically execute 4-6 acting steps per each 
learning step. We sample sequences of length $n=33$ in batches of 4. A moving 
network is unrolled over frames 1-32 while the target network is unrolled over 
frames 2-33.

We allow the agent to be distributed over 
multiple machines each containing action-learner pairs. Each worker downloads the 
newest network parameters before each learning step and sends delta-updates at 
the end of it. Both the network and target 
network are stored on a 
shared parameter server while each machine contains its own local replay 
memory. 
Training is done by downloading a shared network, evaluating local gradients 
and 
sending them to be applied on the shared network. While the agent can also be 
trained on a single machine, in this work we present results of 
training obtained with either 10 or 20 actor-learner workers and one parameter server.
In Figure~\ref{parallelism} (right) we compare resources and runtimes of 
Reactor with related algorithms.\footnote{All results are reported with respect 
to the combined total number of observations obtained over all worker machines.}



\subsubsection{Network architecture}\label{sec:net_arch}
In some domains, such as Atari, it is useful to base decisions on a 
short history of past observations. The two techniques generally used to 
achieve this are frame stacking and recurrent network architectures. We 
chose the latter over the former for reasons 
of implementation simplicity and computational 
efficiency. As the Retrace algorithm requires evaluating action-values over contiguous 
sequences of trajectories, using a recurrent architecture allowed each frame to 
be processed by the convolutional network only once, as opposed to $n$ times 
times if $n$ frame concatenations were used. 

The Reactor architecture uses a recurrent neural network which 
takes an observation $x_t$ as input and produces two outputs:  
categorical action-value distributions $q_i(x_t, a)$ ($i$ here is a bin 
identifier), and policy probabilities $\pi(a|x_t)$. 
We use an architecture inspired by the duelling network architecture 
\citep{wang2015dueling}. We split action-value -distribution logits into 
state-value logits and advantage logits, which 
in turn are connected to the same LSTM network \citep{hochreiter1997long}.
Final action-value logits are produced by summing state- and action-specific 
logits, as in \citet{wang2015dueling}. Finally, a softmax layer on top for each action produces the distributions over discounted 
future returns.

The policy head uses a softmax layer mixed with a fixed 
uniform distribution over actions, where this mixing ratio is a 
hyperparameter \citep[Section 5.1.3]{wiering1999explorations}. Policy and 
Q-networks have separate LSTMs. 
Both LSTMs are connected to a shared 
linear layer which is connected to a shared convolutional neural network 
\citep{krizhevsky2012imagenet}. The precise network
specification is given in Table 
\ref{specs-table} in the Appendix.

Gradients coming from the policy LSTM are blocked and only gradients 
originating from the Q-network LSTM are allowed to back-propagate into the convolutional 
neural network. We block gradients 
from the policy head for increased stability, as this avoids positive feedback 
loops between $\pi$ and $q_i$ caused by shared representations. We used the Adam optimiser \citep{kingma2014adam}, with a learning rate of $5 \times 10^{-5}$ 
and zero momentum because asynchronous updates induce 
implicit momentum \citep{mitliagkas2016asynchrony}. 
Further discussion of hyperparameters and their optimization can be found in Appendix~\ref{sec:par_opt}.

\section{Experimental Results}\label{sec:results}

\begin{figure}
\centering
\vspace{0pt}
\begin{minipage}{.48\textwidth}
  \centerline{\includegraphics[width=1.0\textwidth]{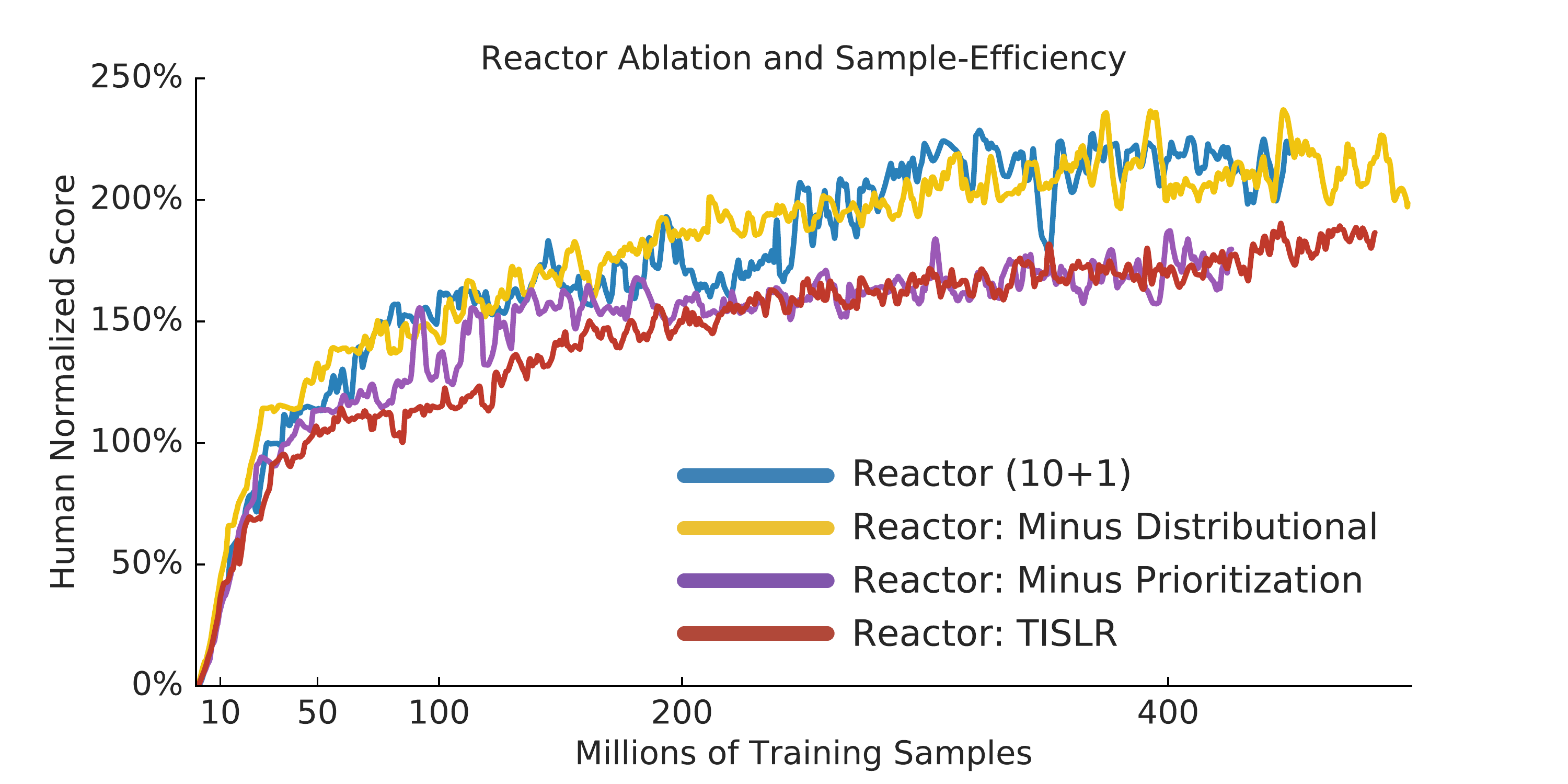}}
\end{minipage}%
\vspace{0pt}
\begin{minipage}{.48\textwidth}
\vspace{0pt}
  \centering
  \includegraphics[width=1.0\textwidth]{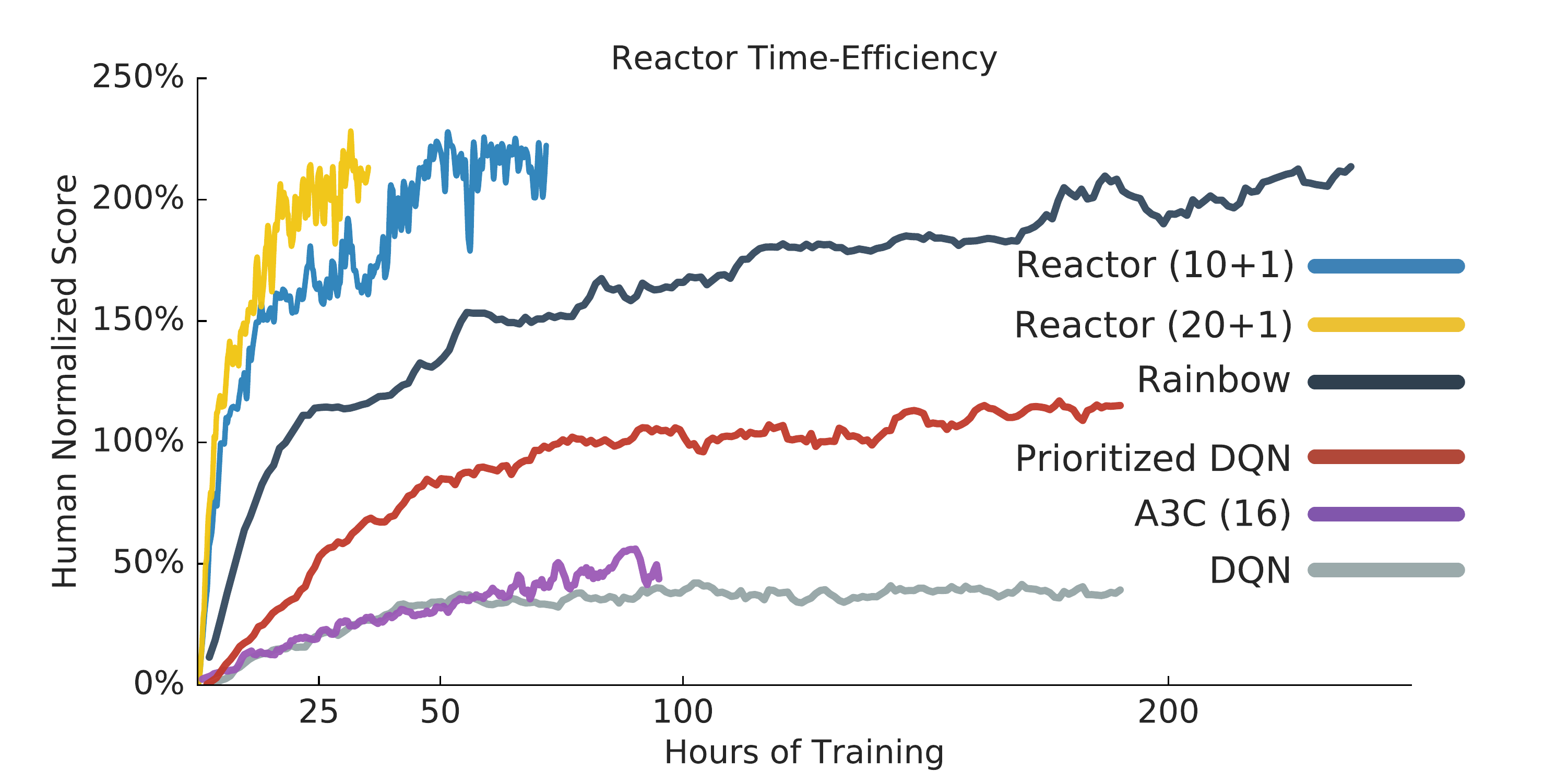}
\end{minipage}
  \caption{
  \small (Left) Reactor performance as various components are removed.
  (Right) Performance comparison as a function of training time in hours. Rainbow learning curve provided by \cite{rainbow}.}
  \label{learningcurves}
\end{figure}
We trained and evaluated Reactor on 57 Atari games \citep{bellemare2013arcade}. 
Figure \ref{learningcurves} compares the performance of Reactor with 
different versions of Reactor each time leaving one of the algorithmic 
improvements out. We can see that each of the algorithmic improvements 
(Distributional retrace, beta-LOO and prioritized replay) 
contributed to the final results. While prioritization was arguably the 
most important component, Beta-LOO clearly
outperformed TISLR algorithm. Although distributional and non-distributional 
versions performed similarly in terms of median human normalized scores, 
distributional version of the 
algorithm generalized better when tested with random human starts (Table 
\ref{comparison-table1}).

\begin{table}[ht]
\centering
\hspace{-3cm}
\begin{minipage}{.55\textwidth}
\begin{center}
\begin{small}
\begin{sc}
\include{large-summary-table}\hfill
\end{sc}
\end{small}
\end{center}
\vspace{-0.2 in}
\caption{\small Random human starts}
\label{comparison-table1}
\end{minipage}
\begin{minipage}{.35\textwidth}
\begin{center}
\begin{small}
\begin{sc}
\hspace{2cm} \include{large-noop-summary-table-expanded}
\end{sc}
\hfill
\end{small}
\end{center}
\vspace{-0.2in}
\caption{\small 30 random no-op starts.}
\label{comparison-table3}
\end{minipage}
\end{table}
\subsection{Comparing to prior work}
We evaluated Reactor with target update frequency $T_{update}=1000$, 
$\lambda=1.0$ and $\beta$-LOO with $\beta=1$ on 
57 Atari games trained on 10 machines in parallel. 
We averaged scores over 200 episodes using 30 random human starts 
and noop starts (Tables \ref{raw-scores-human} and 
\ref{raw-scores-noop} in the Appendix). We calculated mean and 
median human normalised scores across all games. We also ranked 
all algorithms (including random and human scores) for each game and evaluated 
mean rank of each algorithm across all 57 Atari games. We also evaluated mean
Rank and Elo scores for each algorithm for both human and noop start settings. Please refer to Section 
\ref{rank-and-elo} in the Appendix for more details.

Tables \ref{comparison-table1} \& \ref{comparison-table3} compare versions of our algorithm,\footnote{
  \label{reactorfootnote}
  `ND` stands for a non-distributional (i.e. classical) version of Reactor 
  using Retrace \citep{munos2016safe}.
} with several other state-of-art algorithms across 57 Atari games for a 
fixed random seed across all games \citep{bellemare2013arcade}. 
We compare Reactor against are: DQN \citep{mnih15human}, Double 
DQN \citep{van2016deep}, DQN with prioritised experience replay 
\citep{schaul2015prioritized}, dueling architecture and prioritised dueling 
\citep{wang2015dueling}, ACER \citep{wang2017sample}, A3C \citep{mnih2016asynchronous}, and Rainbow \citep{rainbow}. 
Each algorithm was exposed to 200 million 
frames of experience, or 500 million frames when followed by $500\textsc{m}$, and the same pre-processing 
pipeline including 4 action repeats was used as in the original DQN paper \citep{mnih15human}. 

In Table \ref{comparison-table1}, we see that Reactor exceeds the performance of all algorithms across all metrics, 
despite requiring under two days of training. With 500 million frames and four days training we see Reactor's performance 
continue to improve significantly. The difference in time-efficiency is especially apparent when 
comparing Reactor and Rainbow (see Figure~\ref{learningcurves}, right).
Additionally, unlike Rainbow, Reactor does not use Noisy Networks 
\citep{fortunato2017noisy}, which was reported to have contributed 
to the performance gains. When evaluating under the no-op starts regime (Table \ref{comparison-table3}), 
Reactor out performs all methods except for Rainbow. This suggests that Rainbow is more sample-efficient
when training and evaluation regimes match exactly, but may be overfitting to particular trajectories due to 
the significant drop in performance when evaluated on the random human starts.

Regarding ACER, another Retrace-based actor-critic 
architecture, both classical and distributional versions of Reactor (Figure 
\ref{learningcurves}) exceeded the best 
reported median human normalized score of 1.9 with noop starts achieved in 500 
million steps.\footnote{
  \label{acerfootnote}
  Score for ACER in Table 2 was obtained from (Figure 1 in 
  \cite{wang2017sample}), but is not directly comparable due to the authors' 
  use of a cumulative maximization along each learning curve before taking the median.
}
\section{Conclusion}
In this work we presented a new off-policy agent based on Retrace actor-critic 
architecture and show that it achieves similar performance as the 
current state-of-the-art while giving significant real-time performance gains. 
We demonstrate the benefits of each of the suggested algorithmic improvements, 
including Distributional Retrace, beta-LOO policy gradient and contextual 
priority tree.

\bibliography{retrace_agent}
\bibliographystyle{iclr2018_conference}
\normalsize

\onecolumn

\section{Appendix}

\propbias*
\begin{proof}
The bias of $\hat G_{\mbox{\tiny $\beta$-LOO}}$ is
\beqan
\E[\hat G_{\mbox{\tiny $\beta$-LOO}}] - G &=& \sum_a \mu(a) [ \beta(a) (\E[R(a)] -Q(a)) ] \nabla\pi(a) + \sum_{a} Q(a) \nabla\pi(a) - G\\
&=& \sum_a (1-\mu(a)\beta(a)) [Q(a)-Q^{\pi}(a)] \nabla\pi(a)
\eeqan
\end{proof}

\subsection{Hyperparameter optimization}\label{sec:par_opt}
As we believe that algorithms should be robust with respect to the choice of 
hyperparameters, we spent little effort on parameter optimization. In total, we 
explored three distinct values of learning rates and two values of ADAM 
momentum (the default and zero) and two values of $T_{update}$ on a subset of 7 
Atari games without prioritization using non-distributional version of Reactor. 
We later used those values for all experiments. We did not optimize for batch 
sizes and sequence length or any prioritization hyperparamters.

\subsection{Rank and Elo evaluation}
\label{rank-and-elo}

Commonly used mean and median human normalized scores have several 
disadvantages. A mean human normalized score implicitly puts more weight on 
games that computers are good and humans are bad at. Comparing algorithm by a 
mean human normalized score across 57 Atari games is almost equivalent to 
comparing algorithms on a small subset of games close to the median 
and thus dominating the signal.
Typically a set of ten most score-generous games, namely Assault, Asterix, 
Breakout, Demon Attack, Double Dunk, Gopher, Pheonix, Stargunner, Up'n Down and 
Video Pinball can explain more than half of inter-algorithm variance. A median 
human normalized score has the opposite disadvantage by effectively discarding 
very easy and very hard games from the comparison. As typical median human 
normalized scores are within the range of 1-2.5, an algorithm which scores zero 
points on Montezuma's Revenge is evaluated equal to the one which scores 2500 
points, as both performance levels are still below human performance making 
incremental improvements on hard games not being reflected in the overall 
evaluation. In order to address both problem, we also evaluated mean rank and 
Elo metrics for inter-algorithm comparison. Those metrics implicitly assign the 
same weight to each game, and as a result is more sensitive of relative 
performance on very 
hard and easy games: swapping scores of two algorithms on any game would 
result in the change of both mean rank and Elo metrics.

We calculated separate mean rank and Elo scores for each algorithm using 
results of test evaluations with 30 random noop-starts and 30 random human 
starts (Tables \ref{raw-scores-noop} and \ref{raw-scores-human}).
All algorithms were ranked across each game separately, and a mean rank 
was evaluated across 57 Atari games. For Elo score 
evaluation algorithm, $A$ was considered to win over algorithm $B$ if 
it obtained more scores on a given Atari. We produced an empirical 
win-probability matrix by summing wins across all games and used this matrix to 
evaluate Elo scores. A ranking difference of 400 corresponds to the odds 
of winning of 10:1 under the Gaussian assumption.

\subsection{Contextual priority tree}
\label{app-tree-prioritized}

\begin{figure}
\centering
\begin{minipage}{1.0\textwidth}  
  \centering
  \includegraphics[width=0.6\textwidth]{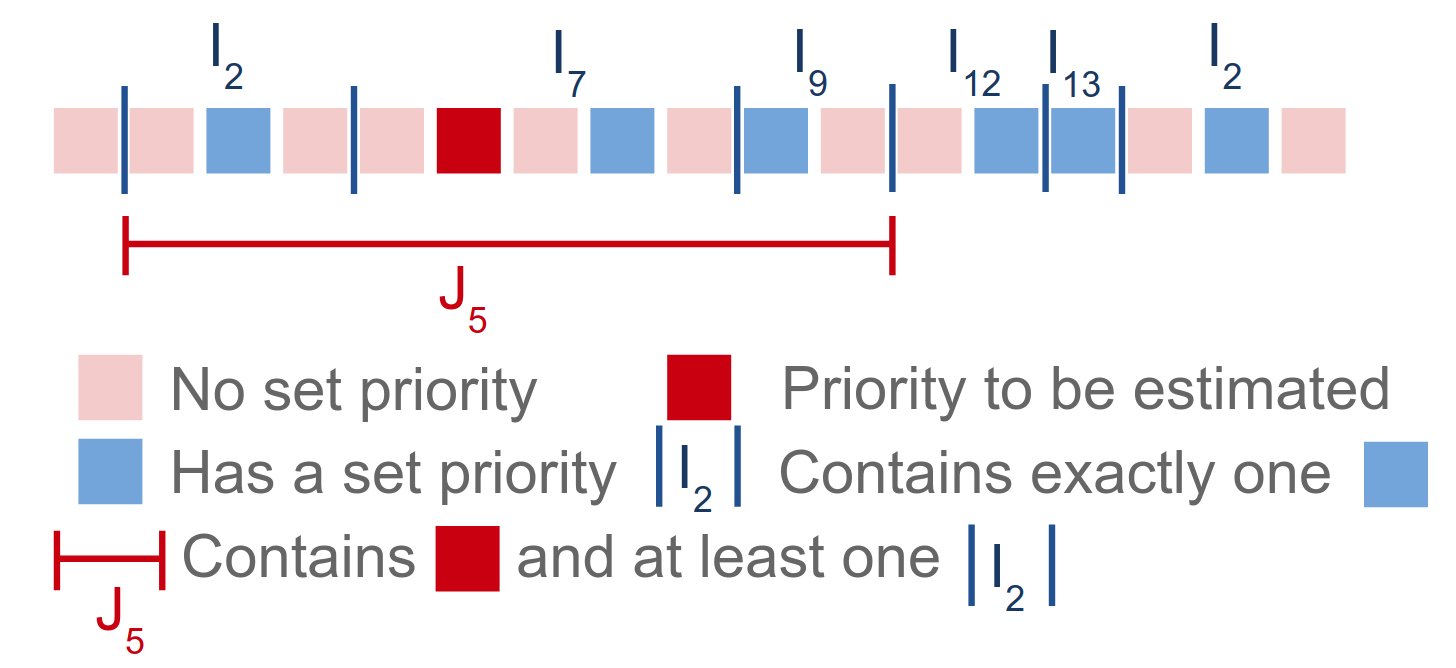}
  \caption{Illustration of Lazy prioritization, where sequences with no 
    explicitly assigned priorities get priorities estimated by a linear 
    combination of nearby assigned priorities. Exact boundaries of blue 
    and red intervals are arbitrary (as long as all conditions described in 
    Section \ref{sec:prioritized_seq} are satisfied) thus leading to many 
    possible algorithms. Each square represents an individual sequence of 
    size 32 (sequences overlap).
    Inverse sizes of blue regions work as local density estimates allowing to 
    produce unbiased priority estimates.}
  \label{lazy-prioritization}
\end{minipage}%
\end{figure}

\begin{figure}
\centering
\begin{minipage}{1.0\textwidth}  
  \centering
  \includegraphics[width=1.0\textwidth]{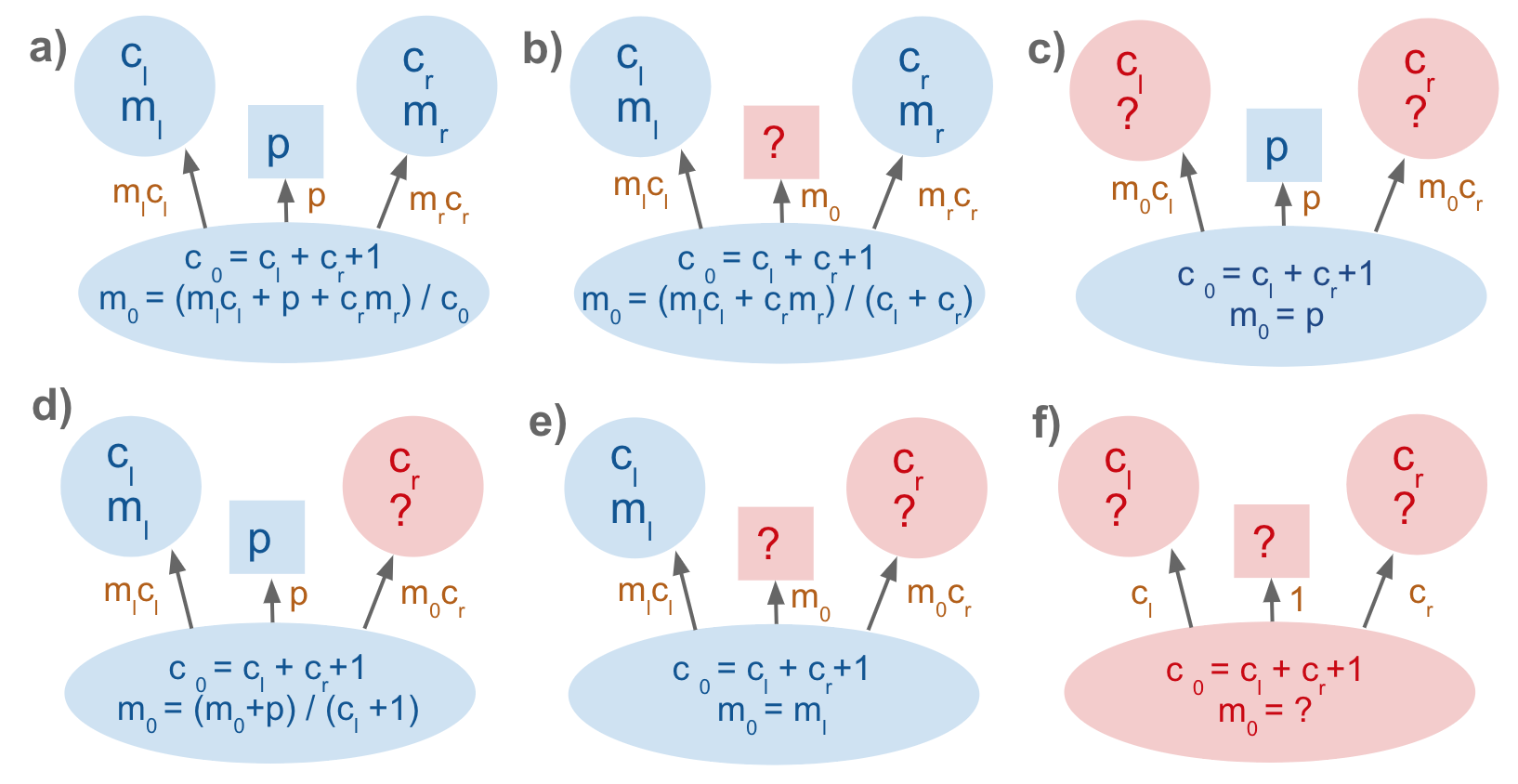}
  \caption{Rules used to evaluate summary statistics on each node of a binary 
    search tree where all sequence keys are kept sorted by temporal order. 
    $c_l$ and $c_r$ are total number of nodes within left and right subtrees. 
    $m_l$ and $m_l$ are estimated mean priorities per node within the subtree. 
    A central square node corresponds to a single key stored within the parent 
    node with its corresponding priority of $p$ (if set) or $?$ if not set. 
    Red subtrees do not have any singe child with a set priority, and a result 
    do not have priority estimates. A red square shows that priority of the key 
    stored within the parent node is not known. Unknown mean priorities is 
    marked by a question mark. Empty child nodes simply behave as if $c=0$ with 
    $p=?$. Rules a-f illustrate how mean values are propagated down from 
    children to parents when priorities are only partially known (rules d and e 
    also apply symmetrically). Sampling is done by going from the root node up 
    the tree by selecting one of the children (or the current key) 
    stochastically proportional to orange proportions. Sampling terminates once 
    the current (square) key is chosen.}
  \label{priority-rule}
\end{minipage}%
\end{figure}

\begin{figure}
\centering
\begin{minipage}{1.0\textwidth}  
  \centering
  \includegraphics[width=0.5\textwidth]{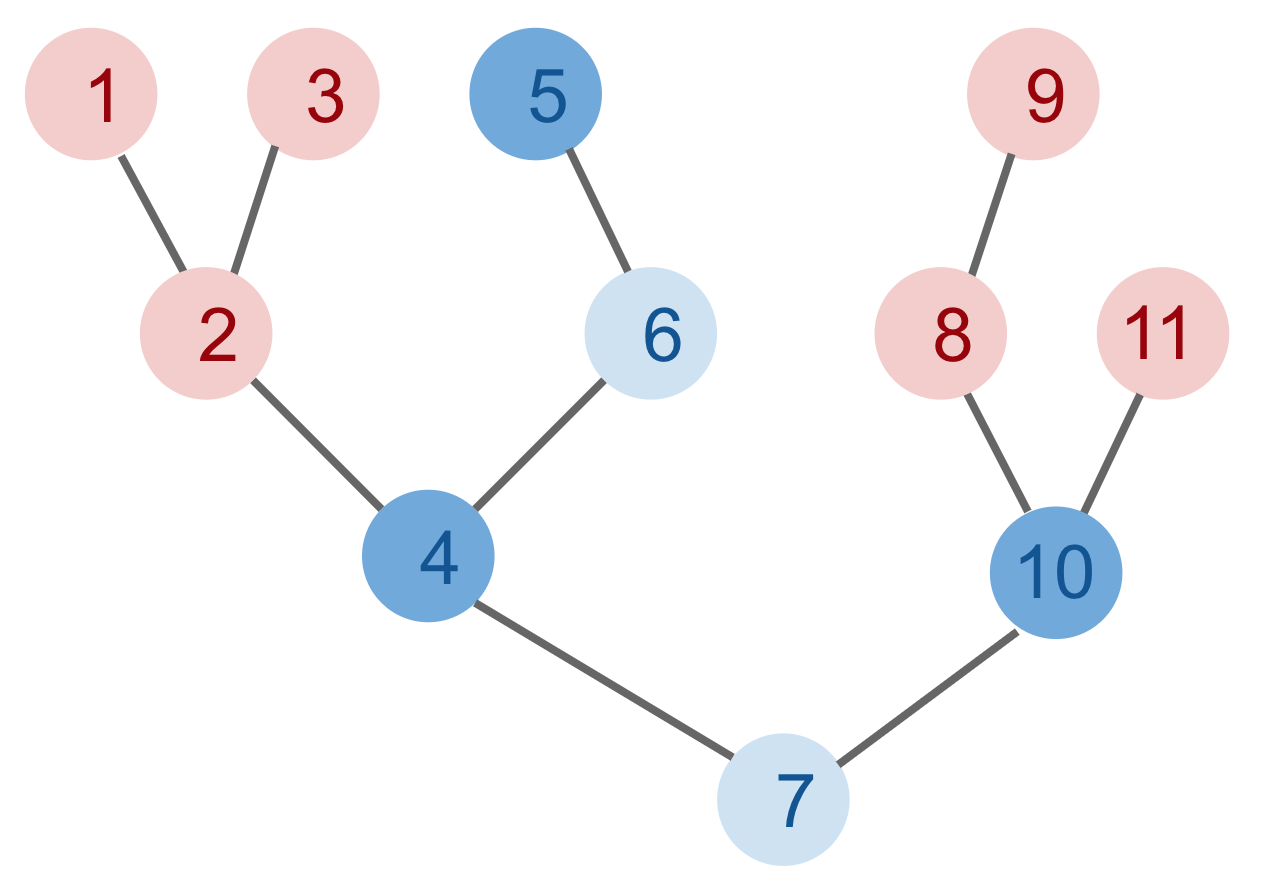}
  \caption{Example of a balanced priority tree. Dark blue nodes contain keys 
    with known priorities, light blue nodes have at least one child with at 
    least a single known priority, while ping nodes do not have any priority 
    estimates. Nodes 1, 2 and 3 will obtain priority estimates equal to $2/3$ 
    of the priority of key 5 and $1/3$ of the priority of node 4. This implies 
    that estimated priorities of keys 1, 2 and 3 are implicitly defined by keys 
    4 and 6. Nodes 8, 9 and 11 are estimated to have the same priority as node 
    10.}
  \label{priority-tree}
\end{minipage}%
\end{figure}

Contextual priority tree is one possible implementation of lazy 
prioritization (Figure \ref{lazy-prioritization}). All sequence 
keys are put into a balanced binary search tree which maintains a temporal 
order. An AVL tree (\cite{velskii1976avl}) was chosen due to the ease of 
implementation and because it is on average more evenly balanced than a 
Red-Black Tree.

Each tree node has up to two children (left and right) and contains currently 
stored key 
and a priority of the key which is either set or is unknown. Some trees may 
only 
have a single child subtree while some may have none. In addition to this 
information, we were tracking other summary statistics at each node which was 
re-evaluated after each tree rotation. The summary statistics was evaluated by 
consuming previously evaluated summary statistics of both children and a 
priority of the key stored within the current node. In particular, we were 
tracking a total number of nodes within each subtree and mean-priority 
estimates 
updated according to rules shown in Figure \ref{priority-rule}. The total 
number of 
nodes within each subtree was always known ($c$ in Figure \ref{priority-rule}), 
while mean priority estimates per key ($m$ in Figure \ref{priority-rule}) could 
either be known or unknown. 

If a mean priority of either one child subtree or a key 
stored within the current node is unknown then it can be estimated to by 
exploiting information coming from another sibling subtree or a priority 
stored within the parent node.

Sampling was done by traversing the tree from the root node up while sampling 
either one of the children subtrees or the currently held key proportionally to 
the total estimated priority masses contained within. The rules used to 
evaluate proportions are shown in orange in Figure 
\ref{priority-rule}. Similarly, probabilities of arbitrary keys can be queried 
by traversing the tree from the root node towards the child node of an interest 
while maintaining a product of probabilities at each branching point. 
Insertion, 
deletion, sampling and probability query operations can be done in O(ln(n)) 
time.

The suggested algorithm has the desired property that it becomes a simple 
proportional sampling algorithm once all the priorities are known. While 
some key priorities are unknown, they are estimated by using nearby known 
key priorities (Figure \ref{priority-tree}).

Each time when a new sequence key is added to the tree, it was set to have an 
unknown priority. Any priority was assigned only after the key got first 
sampled and the corresponding sequence got passed through the learner. When a 
priority of a key is set or updated, the key node is deliberately removed from 
and placed back to the tree in order to become a leaf-node. This helped to 
set priorities of nodes in the immediate vicinity more accurately by using the 
freshest information available.

\subsection{Network architecture}
The value of 
$\epsilon=0.01$ is the minimum probability of choosing a random action and it 
is hard-coded into the policy network. Figure \ref{netarch} shows the overall 
network topology while Table \ref{specs-table} specifies network layer sizes.
\newpage
\begin{figure}[ht]
\centering
\includegraphics[width=0.5\textwidth]{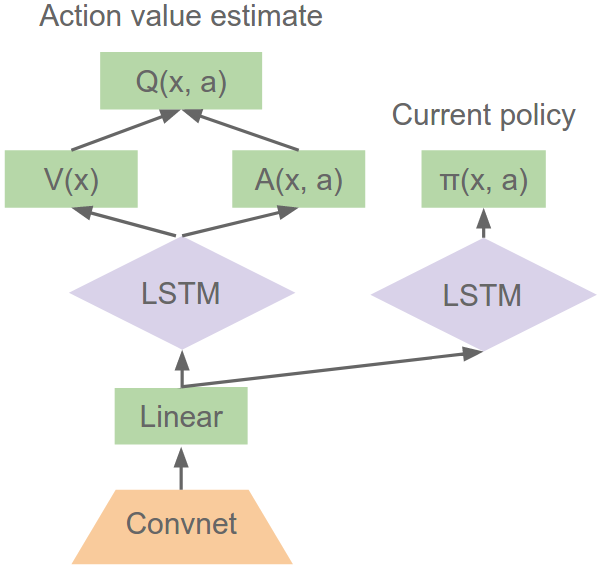}
  \caption{Network architecture.}
  \label{netarch}
\end{figure}
\begin{table}[ht]
\caption{Specification of the neural network used (illustrated in Figure 
\ref{netarch})}
\label{specs-table}
\vskip 0.1in
\begin{center}
\begin{sc}
\begin{tabular}{| c | c | c | c | c |}
\hline
Layer & Input & \multicolumn{3}{|c|}{ Parameters} \\
      & size  & \multicolumn{3}{|c|}{ } \\
\hline
Convolutional & & kernel & output & strides\\
              & & width & channels & \\
\hline
Conv 1 & [84, 84, 1] & [8, 8] & 16 & 4 \\
ConcatReLU & [20, 20, 16] & & & \\
Conv 2 & [20, 20, 32] & [4, 4] & 32 & 2 \\
ConcatReLU & [9, 9, 32] & & & \\
Conv 3 & [9, 9, 64] & [3, 3] & 32 & 1 \\
ConcatReLU & [7, 7, 32] & & & \\
\hline
Fully connected & & \multicolumn{3}{|c|}{ Output size} \\
\hline
Linear & [7, 7, 64] & \multicolumn{3}{|c|}{ 128 }\\
ConcatReLU & [128] & \multicolumn{3}{|c|}{} \\
\hline
Recurrent & & \multicolumn{3}{|c|}{ Output size} \\
$\pi$ & & \multicolumn{3}{|c|}{ } \\
\hline
LSTM & [256] & \multicolumn{3}{|c|}{ 128 }\\
Linear & [128] & \multicolumn{3}{|c|}{ 32 }\\
ConcatReLU & [32] & \multicolumn{3}{|c|}{ }\\
Linear & [64] & \multicolumn{3}{|c|}{ \#actions }\\
Softmax & [\#actions] & \multicolumn{3}{|c|}{ \#actions }\\
x(1-$\epsilon$)+$\epsilon$/\#actions & [\#actions] & \multicolumn{3}{|c|}{ 
\#actions }\\
\hline
Recurrent $Q$ & & \multicolumn{3}{|c|}{ Output size} \\
\hline
LSTM & [256] & \multicolumn{3}{|c|}{ 128 }\\
\hline
Value logit head & & \multicolumn{3}{|c|}{  Output size} \\
\hline
Linear & [128] & \multicolumn{3}{|c|}{ 32 }\\
ConcatReLU & [32] & \multicolumn{3}{|c|}{ }\\
Linear & [64] & \multicolumn{3}{|c|}{ \#bins }\\
\hline
Advantage logit head & & \multicolumn{3}{|c|}{ \#actions $\times$ \#bins} \\
\hline
Linear & [128] & \multicolumn{3}{|c|}{ 32 }\\
ConcatReLU & [32] & \multicolumn{3}{|c|}{ }\\
\hline
\end{tabular}
\end{sc}
\end{center}
\vskip -0.1in
\end{table}
\newpage
\subsection{Comparisons with Rainbow}
In this section we compare Reactor with the recently published Rainbow agent 
\citep{rainbow}. While ACER is the most closely related algorithmically, Rainbow 
is most closely related in terms of performance and thus a deeper understanding 
of the trade-offs between Rainbow and Reactor may benefit interested readers. 
There are many architectural and algorithmic differences between Rainbow and 
Reactor. We will therefore begin by highlighting where they agree. Both use a 
categorical action-value distribution critic 
\citep{bellemare2017distributional}, factored into state and state-action logits 
\citep{wang2015dueling},
\begin{equation*}
q_i(x, a) = \frac{l_i(x, a)}{\sum_j l_j(x, a)},\ \quad l_i(x, a) = l_i(x) + l_i(x, a) - \frac{1}{|\mathcal{A}|} \sum_{b \in \mathcal{A}} l_i(x, b).
\end{equation*}
Both use prioritized replay, and finally, both perform $n$-step Bellman updates.

Despite these similarities, Reactor and Rainbow are fundamentally different 
algorithms and are based upon different lines of research. While Rainbow uses 
Q-Learning and is based upon DQN \citep{mnih15human}, Reactor is an actor-critic 
algorithm most closely based upon A3C \citep{mnih2016asynchronous}. Each 
inherits some design choices from their predecessors, and we have not performed 
an extensive ablation comparing these various differences. Instead, we will 
discuss four of the differences we believe are important but less obvious.

First, the network structures are substantially different. Rainbow uses noisy 
linear layers and ReLU activations throughout the network, whereas Reactor uses 
standard linear layers and concatenated ReLU activations throughout. To overcome 
partial observability, Rainbow, inheriting this choice from DQN, uses 
\textit{frame stacking}. On the other hand, Reactor, inheriting its choice from 
A3C, uses LSTMs after the convolutional layers of the network. It is also 
difficult to directly compare the number of parameters in each network because 
the use of noisy linear layers doubles the number of parameters, although half 
of these are used to control noise, while the LSTM units in Reactor require more 
parameters than a corresponding linear layer would.

Second, both algorithms perform $n$-step updates, however, the Rainbow $n$-step 
update does not use any form of off-policy correction. Because of this, Rainbow 
is restricted to using only small values of $n$ (e.g. $n = 3$) because larger 
values would make sequences more off-policy and hurt performance. By comparison, 
Reactor uses our proposed distributional Retrace algorithm for off-policy 
correction of $n$-step updates. This allows the use of larger values of $n$ 
(e.g. $n = 33$) without loss of performance.

Third, while both agents use prioritized replay buffers 
\citep{schaul16prioritized}, they each store different information and 
prioritize using different algorithms. Rainbow stores a tuple containing the 
state $x_{t-1}$, action $a_{t-1}$, sum of $n$ discounted rewards 
$\sum_{k=0}^{n-1} r_{t+k} \prod_{m=0}^{k-1} \gamma_{t+m}$, product of $n$ 
discount factors $\prod_{k=0}^{n-1} \gamma_{t+k}$, and next-state $n$ steps away 
$x_{t+n-1}$. Tuples are prioritized based upon the last observed TD error, and 
inserted into replay with a maximum priority. Reactor stores length $n$ 
sequences of tuples $(x_{t-1}, a_{t-1}, r_t, \gamma_t)$ and also prioritizes 
based upon the observed TD error. However, when inserted into the buffer the 
priority is instead inferred based upon the known priorities of neighboring 
sequences. This priority inference was made efficient using the previously 
introduced contextual priority tree, and anecdotally we have seen it improve 
performance over a simple maximum priority approach.

Finally, the two algorithms have different approaches to exploration. 
Rainbow, unlike DQN, does not use $\epsilon$-greedy exploration, but instead 
replaces all linear layers with noisy linear layers which induce randomness 
throughout the network. This method, called Noisy Networks 
\citep{fortunato2017noisy}, creates an adaptive exploration integrated into the 
agent's network. Reactor does not use noisy networks, but instead uses the same 
entropy cost method used by A3C and many others \citep{mnih2016asynchronous}, 
which penalizes deterministic policies thus encouraging indifference between 
similarly valued actions. Because Rainbow can essentially learn not to explore, 
it may learn to become entirely greedy in the early parts of the episode, while 
still exploring in states not as frequently seen. In some sense, this is 
precisely what we want from an exploration technique, but it may also lead to 
highly deterministic trajectories in the early part of the episode and an 
increase in overfitting to those trajectories. We hypothesize that this may be 
the explanation for the significant difference in Rainbow's performance between 
evaluation under no-op and random human starts, and why Reactor does not show 
such a large difference.

\subsection{Atari results}

\begin{table}[h]
\caption{Scores for each game evaluated with 30 random human starts. Reactor 
was evaluated by averaging scores over 200 episodes. All scores (except for 
Reactor) were taken from \cite{wang2015dueling}, \cite{mnih2016asynchronous} 
and 
\cite{rainbow}.}
\label{raw-scores-human}
\begin{scriptsize}
\begin{sc}
\tiny
\include{raw-scores}
\end{sc}
\end{scriptsize}
\end{table}

\begin{table}[h]
\caption{Scores for each game evaluated with 30 random noop starts. Reactor 
was evaluated by averaging scores over 200 episodes. All scores (except for 
Reactor) were taken from \cite{wang2015dueling} and \cite{rainbow}.}
\label{raw-scores-noop}
\begin{scriptsize}
\begin{sc}
\tiny
\include{raw-scores-noop}
\end{sc}
\end{scriptsize}
\end{table}

\end{document}

%% file: large-summary-table.tex
\begin{tabular}{| c | c | c | c |}
\hline
Algorithm & Normalized & Mean& Elo \\
 & scores & rank &\\
\hline
Random & 0.00 & 11.65 & -563\\
Human & 1.00 & 6.82 & 0\\
DQN & 0.69 & 9.05 & -172\\
DDQN & 1.11 & 7.63 & -58\\
Duel & 1.17 & 6.35 & 32\\
Prior & 1.13 & 6.63 & 13\\
Prior. Duel. & 1.15 & 6.25 & 40\\
A3C LSTM & 1.13 & 6.30 & 37\\
Rainbow & 1.53 & 4.18 & 186\\
Reactor ND $^{\ref{reactorfootnote}}$ & 1.51 & 4.98 & 126\\
Reactor & 1.65 & 4.58 & 156\\
Reactor 500m & {\bf 1.82} & {\bf 3.65} & {\bf 227}\\
\hline
\end{tabular}

%% file: large-noop-summary-table-expanded.tex
\begin{tabular}{| c | c | c | c |}
\hline
Algorithm & Normalized & Mean& Elo \\
 & scores & rank &\\
\hline
Random & 0.00 & 10.93 & -673\\
Human & 1.00 & 6.89 & 0\\
DQN & 0.79 & 8.65 & -167\\
DDQN & 1.18 & 7.28 & -27\\
Duel & 1.51 & 5.19 & 143\\
Prior & 1.24 & 6.11 & 70\\
Prior. Duel. & 1.72 & 5.44 & 126\\
ACER$^{\ref{acerfootnote}}$ 500m & 1.9 & - & - \\
Rainbow & {\bf 2.31} & 3.63 & 270\\
Reactor ND $^{\ref{reactorfootnote}}$ & 1.80 & 4.53 & 195\\
Reactor & 1.87 & 4.46 & 196\\
Reactor 500m & 2.30 & {\bf 3.47} & {\bf 280}\\
\hline
\end{tabular}

%% file: raw-scores-noop.tex
\hspace{-40px}\begin{tabular}{| p{2.4cm} p{0.84cm} p{0.84cm} p{0.84cm} p{0.84cm} p{0.84cm} p{0.84cm} p{0.84cm} p{0.84cm} p{0.84cm} p{0.84cm} p{0.84cm} |}
\hline
Game \/ Agent  & Random & Human & DQN & DDQN & Duel & Prior & Prior. Duel. & Rainbow & Reactor ND $^{\ref{reactorfootnote}}$ & Reactor & Reactor 500m\\
\hline
Alien & 227.8 & 7127.7 & 1620.0 & 3747.7 & 4461.4 & 4203.8 & 3941.0 & 9491.7 & 4199.4 & 6482.1 & {\bf 12689.1}\\
Amidar & 5.8 & 1719.5 & 978.0 & 1793.3 & 2354.5 & 1838.9 & 2296.8 & {\bf 5131.2} & 1546.8 & 833.0 & 1015.8\\
Assault & 222.4 & 742.0 & 4280.4 & 5393.2 & 4621.0 & 7672.1 & 11477.0 & 14198.5 & {\bf 17543.8} & 11013.5 & 8323.3\\
Asterix & 210.0 & 8503.3 & 4359.0 & 17356.5 & 28188.0 & 31527.0 & 375080.0 & {\bf 428200.3} & 16121.0 & 36238.5 & 205914.0\\
Asteroids & 719.1 & {\bf 47388.7} & 1364.5 & 734.7 & 2837.7 & 2654.3 & 1192.7 & 2712.8 & 4467.4 & 2780.4 & 3726.1\\
Atlantis & 12850.0 & 29028.1 & 279987.0 & 106056.0 & 382572.0 & 357324.0 & 395762.0 & 826659.5 & {\bf 968179.5} & 308258.0 & 302831.0\\
Bank Heist & 14.2 & 753.1 & 455.0 & 1030.6 & {\bf 1611.9} & 1054.6 & 1503.1 & 1358.0 & 1236.8 & 988.7 & 1259.7\\
Battlezone & 2360.0 & 37187.5 & 29900.0 & 31700.0 & 37150.0 & 31530.0 & 35520.0 & 62010.0 & {\bf 98235.0} & 61220.0 & 64070.0\\
Beam Rider & 363.9 & 16926.5 & 8627.5 & 13772.8 & 12164.0 & 23384.2 & {\bf 30276.5} & 16850.2 & 8811.8 & 8566.5 & 11033.4\\
Berzerk & 123.7 & 2630.4 & 585.6 & 1225.4 & 1472.6 & 1305.6 & {\bf 3409.0} & 2545.6 & 1515.7 & 1641.4 & 2303.1\\
Bowling & 23.1 & {\bf 160.7} & 50.4 & 68.1 & 65.5 & 47.9 & 46.7 & 30.0 & 59.3 & 75.4 & 81.0\\
Boxing & 0.1 & 12.1 & 88.0 & 91.6 & 99.4 & 95.6 & 98.9 & 99.6 & {\bf 99.7} & 99.4 & 99.4\\
Breakout & 1.7 & 30.5 & 385.5 & 418.5 & 345.3 & 373.9 & 366.0 & 417.5 & 509.5 & {\bf 518.4} & 514.8\\
Centipede & 2090.9 & {\bf 12017.0} & 4657.7 & 5409.4 & 7561.4 & 4463.2 & 7687.5 & 8167.3 & 7267.2 & 3402.8 & 3422.0\\
Chopper Command & 811.0 & 7387.8 & 6126.0 & 5809.0 & 11215.0 & 8600.0 & 13185.0 & 16654.0 & 19901.5 & 37568.0 & {\bf 107779.0}\\
Crazy Climber & 10780.5 & 35829.4 & 110763.0 & 117282.0 & 143570.0 & 141161.0 & 162224.0 & 168788.5 & 173274.0 & 194347.0 & {\bf 236422.0}\\
Defender & 2874.5 & 18688.9 & 23633.0 & 35338.5 & 42214.0 & 31286.5 & 41324.5 & 55105.0 & 181074.3 & 113128.0 & {\bf 223025.0}\\
Demon Attack & 152.1 & 1971.0 & 12149.4 & 58044.2 & 60813.3 & 71846.4 & 72878.6 & 111185.2 & {\bf 122782.5} & 100189.0 & 115154.0\\
Double Dunk & -18.6 & -16.4 & -6.6 & -5.5 & 0.1 & 18.5 & -12.5 & -0.3 & 23.0 & 11.4 & {\bf 23.0}\\
Enduro & 0.0 & 860.5 & 729.0 & 1211.8 & 2258.2 & 2093.0 & {\bf 2306.4} & 2125.9 & 2211.3 & 2230.1 & 2224.2\\
Fishing Derby & -91.7 & -38.7 & -4.9 & 15.5 & {\bf 46.4} & 39.5 & 41.3 & 31.3 & 33.1 & 23.2 & 30.4\\
Freeway & 0.0 & 29.6 & 30.8 & 33.3 & 0.0 & 33.7 & 33.0 & {\bf 34.0} & 22.3 & 31.4 & 31.5\\
Frostbite & 65.2 & 4334.7 & 797.4 & 1683.3 & 4672.8 & 4380.1 & 7413.0 & {\bf 9590.5} & 7136.7 & 8042.1 & 7932.2\\
Gopher & 257.6 & 2412.5 & 8777.4 & 14840.8 & 15718.4 & 32487.2 & {\bf 104368.2} & 70354.6 & 36279.1 & 69135.1 & 89851.0\\
Gravitar & 173.0 & {\bf 3351.4} & 473.0 & 412.0 & 588.0 & 548.5 & 238.0 & 1419.3 & 1804.8 & 1073.8 & 2041.8\\
H.E.R.O. & 1027.0 & 30826.4 & 20437.8 & 20130.2 & 20818.2 & 23037.7 & 21036.5 & {\bf 55887.4} & 27833.0 & 35542.2 & 43360.4\\
Ice Hockey & -11.2 & 0.9 & -1.9 & -2.7 & 0.5 & 1.3 & -0.4 & 1.1 & {\bf 15.7} & 3.4 & 10.7\\
James Bond 007 & 29.0 & 302.8 & 768.5 & 1358.0 & 1312.5 & 5148.0 & 812.0 & {\bf 19809.0} & 14524.0 & 7869.2 & 16056.2\\
Kangaroo & 52.0 & 3035.0 & 7259.0 & 12992.0 & 14854.0 & {\bf 16200.0} & 1792.0 & 14637.5 & 13349.0 & 10484.5 & 11266.5\\
Krull & 1598.0 & 2665.5 & 8422.3 & 7920.5 & {\bf 11451.9} & 9728.0 & 10374.4 & 8741.5 & 10237.8 & 9930.8 & 9896.0\\
Kung-Fu Master & 258.5 & 22736.3 & 26059.0 & 29710.0 & 34294.0 & 39581.0 & 48375.0 & 52181.0 & 61621.5 & 59799.5 & {\bf 65836.5}\\
Montezuma's Revenge & 0.0 & {\bf 4753.3} & 0.0 & 0.0 & 0.0 & 0.0 & 0.0 & 384.0 & 0.0 & 2643.5 & 2643.5\\
Ms. Pac-Man & 307.3 & {\bf 6951.6} & 3085.6 & 2711.4 & 6283.5 & 6518.7 & 3327.3 & 5380.4 & 4416.9 & 2724.3 & 3749.2\\
Name This Game & 2292.3 & 8049.0 & 8207.8 & 10616.0 & 11971.1 & 12270.5 & {\bf 15572.5} & 13136.0 & 12636.5 & 9907.2 & 9543.8\\
Phoenix & 761.4 & 7242.6 & 8485.2 & 12252.5 & 23092.2 & 18992.7 & 70324.3 & {\bf 108528.6} & 10261.4 & 40092.2 & 46536.4\\
Pitfall!  & -229.4 & {\bf 6463.7} & -286.1 & -29.9 & 0.0 & -356.5 & 0.0 & 0.0 & -3.7 & -3.5 & -8.9\\
Pong & -20.7 & 14.6 & 19.5 & 20.9 & {\bf 21.0} & 20.6 & 20.9 & 20.9 & 20.7 & 20.7 & 20.6\\
Private Eye & 24.9 & {\bf 69571.3} & 146.7 & 129.7 & 103.0 & 200.0 & 206.0 & 4234.0 & 15198.0 & 15177.1 & 15188.8\\
Q*bert & 163.9 & 13455.0 & 13117.3 & 15088.5 & 19220.3 & 16256.5 & 18760.3 & {\bf 33817.5} & 21222.5 & 22956.5 & 21509.2\\
River Raid & 1338.5 & 17118.0 & 7377.6 & 14884.5 & 21162.6 & 14522.3 & 20607.6 & {\bf 22920.8} & 16957.3 & 16608.3 & 17380.7\\
Road Runner & 11.5 & 7845.0 & 39544.0 & 44127.0 & 69524.0 & 57608.0 & 62151.0 & 62041.0 & 66790.5 & 71168.0 & {\bf 111310.0}\\
Robotank & 2.2 & 11.9 & 63.9 & 65.1 & 65.3 & 62.6 & 27.5 & 61.4 & {\bf 71.8} & 68.5 & 70.4\\
Seaquest & 68.4 & 42054.7 & 5860.6 & 16452.7 & {\bf 50254.2} & 26357.8 & 931.6 & 15898.9 & 5071.6 & 8425.8 & 20994.1\\
Skiing & -17098.1 & {\bf -4336.9} & -13062.3 & -9021.8 & -8857.4 & -9996.9 & -19949.9 & -12957.8 & -10632.9 & -10753.4 & -10870.6\\
Solaris & 1236.3 & {\bf 12326.7} & 3482.8 & 3067.8 & 2250.8 & 4309.0 & 133.4 & 3560.3 & 2236.0 & 2760.0 & 2099.6\\
Space Invaders & 148.0 & 1668.7 & 1692.3 & 2525.5 & 6427.3 & 2865.8 & 15311.5 & {\bf 18789.0} & 2387.1 & 2448.6 & 10153.9\\
Stargunner & 664.0 & 10250.0 & 54282.0 & 60142.0 & 89238.0 & 63302.0 & 125117.0 & {\bf 127029.0} & 48942.0 & 70038.0 & 79521.5\\
Surround & -10.0 & 6.5 & -5.6 & -2.9 & 4.4 & 8.9 & 1.2 & {\bf 9.7} & 0.9 & 6.7 & 7.0\\
Tennis & -23.8 & -8.3 & 12.2 & -22.8 & 5.1 & 0.0 & 0.0 & 0.0 & 23.4 & 23.3 & {\bf 23.6}\\
Time Pilot & 3568.0 & 5229.2 & 4870.0 & 8339.0 & 11666.0 & 9197.0 & 7553.0 & 12926.0 & 18871.5 & {\bf 19401.0} & 18841.5\\
Tutankham & 11.4 & 167.6 & 68.1 & 218.4 & 211.4 & 204.6 & 245.9 & 241.0 & 263.2 & 272.6 & {\bf 275.4}\\
Up'n Down & 533.4 & 11693.2 & 9989.9 & 22972.2 & 44939.6 & 16154.1 & 33879.1 & 125754.6 & {\bf 194989.5} & 64354.2 & 70790.4\\
Venture & 0.0 & 1187.5 & 163.0 & 98.0 & 497.0 & 54.0 & 48.0 & 5.5 & 0.0 & 1597.5 & {\bf 1653.5}\\
Video Pinball & 16256.9 & 17667.9 & 196760.4 & 309941.9 & 98209.5 & 282007.3 & 479197.0 & {\bf 533936.5} & 261720.2 & 469366.0 & 496101.0\\
Wizard of Wor & 563.5 & 4756.5 & 2704.0 & 7492.0 & 7855.0 & 4802.0 & 12352.0 & 17862.5 & 18484.0 & 13170.5 & {\bf 19530.5}\\
Yars' Revenge & 3092.9 & 54576.9 & 18098.9 & 11712.6 & 49622.1 & 11357.0 & 69618.1 & 102557.0 & 109607.5 & 102760.0 & {\bf 148855.0}\\
Zaxxon & 32.5 & 9173.3 & 5363.0 & 10163.0 & 12944.0 & 10469.0 & 13886.0 & 22209.5 & 16525.0 & 25215.5 & {\bf 27582.5}\\
\hline
\end{tabular}